%% file: main.tex
\documentclass[acmsmall]{acmart}

\AtBeginDocument{%
  }

\usepackage{graphicx} 

\usepackage[utf8]{inputenc}

\usepackage{todonotes}
\usepackage{amsmath,amssymb,amsfonts, amsthm}
\usepackage{xcolor}
\usepackage{multirow}
\usepackage{subcaption}
\usepackage[ruled, linesnumbered]{algorithm2e}
\let\oldnl\nl
\newcommand{\nonl}{\renewcommand{\nl}{\let\nl\oldnl}}
\usepackage{algpseudocode}
\usepackage{comment}
\usepackage{multicol}
\usepackage{vwcol}
\usepackage{tikz}
\usepackage[most]{tcolorbox}
\usepackage{todonotes}
\usepackage{mathtools}

\usepackage{caption}
\usepackage{float}
\usepackage{wrapfig,booktabs}

\usetikzlibrary{fit}
\usetikzlibrary{patterns}
\newtheorem{exmp}{Example}

\newtheorem{prob}{Problem Statement}

\usepackage{enumitem}
\usepackage{comment}
\usepackage{todonotes}
\usepackage{url}

\usepackage{stmaryrd} 

\usepackage[normalem]{ulem} 

\usepackage[autostyle, english = american]{csquotes}
\MakeOuterQuote{"}

\setlist{nolistsep}
\algnewcommand{\LeftComment}[1]{\Statex \(\triangleright\) #1}

\begin{document}

\title{Exploring Inevitable Waypoints for Unsolvability Explanation in Hybrid Planning Problems}
%
%
\author{Mir Md Sajid Sarwar}
\affiliation{%
  \institution{School of Mathematical and Computational Sciences, Indian Association for the Cultivation of Science, Kolkata}
  \country{India}
  }
 \email{mcsss2275@iacs.res.in}

\author{Rajarshi Ray}
\affiliation{%
  \institution{School of Mathematical and Computational Sciences, Indian Association for the Cultivation of Science, Kolkata}
  \country{India}
  }
\email{rajarshi.ray@iacs.res.in}


\begin{abstract}
Explaining unsolvability of planning problems is of significant research interest in Explainable AI Planning. A number of research efforts on generating explanations of solutions to planning problems have been reported in AI planning literature. However, explaining the unsolvability of planning problems remains a largely open and understudied problem. A widely practiced approach to plan generation and automated problem solving, in general, is to decompose tasks into sub-problems that help progressively converge towards the goal. In this paper, we propose to adopt the same philosophy of sub-problem identification as a mechanism for analyzing and explaining unsolvability of planning problems in hybrid systems. In particular, for a given unsolvable planning problem, we propose to identify common waypoints, which are universal obstacles to plan existence, in other words, they appear on every plan from the source to the planning goal. This work envisions such waypoints as sub-problems of the planning problem and the unreachability of any of these waypoints as an explanation for the unsolvability of the original planning problem. We propose a novel method of waypoint identification by casting the problem as an instance of the longest common subsequence problem, a widely popular problem in computer science, typically considered as an illustrative example for the dynamic programming paradigm. 
Once the waypoints are identified, we perform symbolic reachability analysis on them to identify the earliest unreachable waypoint and report it as the explanation of unsolvability. We present experimental results on unsolvable planning problems in hybrid domains. 
\end{abstract}



\maketitle

\input{intoduction}
\input{motivation}
\input{problem-statement}
\input{methodology}
\input{evaluation}

\input{related_works}

\input{conclusion}
\clearpage
\bibliographystyle{plain}
\bibliography{mybib}
\clearpage
\input{Appendix}

\end{document}

%% file: intoduction.tex
\section{Introduction}
In human-computer interaction (HCI), humans engage in trustworthy collaborations with autonomous agents, and one of the precursors of such collaborations is that an autonomous agent must explain the rationale behind its decision to the human. With the emergence of artificial intelligence (AI) and the multitude of application domains where AI planning is being envisioned to replace plans generated by humans, Explainable AI Planning (XAIP)~\cite{DBLP:journals/corr/abs-1709-10256,DBLP:conf/rweb/HoffmannM19} has emerged as an important connection between HCI and AI for designing explainable systems that bridges the gap between theoretical and algorithmic planning and real-world applications \cite{ijcai2020p669}.
While there has been a lot of research on generating explanations to planning problems, most of the earlier works in explanation generation have focused on explaining why a given plan or action was chosen \cite{DBLP:conf/ijcai/ChakrabortiSZK17,DBLP:conf/aips/ChakrabortiKSSK19,DBLP:journals/jair/KrarupKMLC021,DBLP:journals/tecs/SarwarRB23}. However, explaining the unsolvability of a given planning problem remains a largely open and understudied problem.
The recent works that focus on explaining the unsolvability of planning problems have primarily concentrated on generating certificates or proofs of unsolvability \cite{DBLP:conf/aips/ErikssonRH17,DBLP:conf/aips/ErikssonRH18}, or on identifying counterfactual alterations to the original planning task to make it solvable, often referred to as "excuses" \cite{DBLP:conf/aips/GobelbeckerKEBN10}. These approaches, which are more oriented towards automatic verification, may fall short in adequately explaining unsolvability in complex planning domains.

A well-known insight into human thinking and problem-solving is that humans tend to decompose a problem into sub-problems that help in progressively converging towards the goal. Many AI systems mimic this notion in the way they solve problems. For instance, the main feature of the pioneering automated theorem prover, \emph{logic theorist}, is the use of problem-subproblem hierarchy \cite{LT-1956}. An innovative technique for identification of sub-problems relevant for explaining unsolvability of a planning problem in domains with discrete dynamics has been proposed in \cite{DBLP:conf/ijcai/SreedharanSSK19}. 
In this paper, we propose to adopt the same philosophy of sub-problem identification as an efficient mechanism for analyzing and explaining unsolvability of planning problems in hybrid domains, domains with a combination of discrete and continuous dynamics. In particular, for a given unsolvable planning problem, we propose to identify a sequence of waypoints, which are universal obstacles to plan existence, in other words, they appear on every path on every plan from the source to the planning goal. This work envisions such waypoints as sub-problems of the planning problem and the unreachability of any of these waypoints as an explanation for the unsolvability of the original problem at hand. We propose a novel method of waypoint identification by casting the problem as an instance of the longest common subsequence problem, a widely popular problem in computer science, typically considered as an illustrative example for the dynamic programming paradigm. Once the waypoints are identified, we perform symbolic reachability analysis on them to identify the earliest unreachable waypoint and report it as the explanation of unsolvability. We present experimental results on unsolvable planning problems in hybrid domains. In summary, the key contributions of this paper are:
\begin{enumerate}
 \item[(a)] A proposal for an artifact for explaining unsolvability of hybrid planning problem based on the identification of inevitable waypoints.
 \item[(b)] A method to generate the explanation artifact by casting it as an instance of the longest common subsequence problem, and subsequently using symbolic reachability analysis on the hybrid automaton.
 
\end{enumerate}

The rest of this paper is organized as follows. In Section~\ref{motive}, we present a motivating example to demonstrate this work. Section~\ref{problem-overview} provides a background and problem overview.  
Section~\ref{methodology} illustrates our methodology and the framework of explanation. Section~\ref{eval} discusses implementation and results. Section~\ref{related_works} presents related literature. Finally, Section~\ref{conclusion} summarizes the contributions and the findings of this work and discusses possible future directions.

%% file: motivation.tex
\section{Motivating Example} \label{motive}
In this section, we present a motivating example in the context of a planning problem for a planetary rover that explores a planetary site and collects samples for experiments. The agent (autonomous battery powered rover) possesses knowledge about the topography of the exploration site as a planar grid.
Figure~\ref{fig:rover-domain} shows the topography as a 5$\times$5 grid of cells. The rover is initially positioned at cell 11 and there is a base-station at cell 25. The task of the rover is to collect soil and rock samples from designated sites, in cells 1 and 14 respectively and then reach the base-station. Mountainous regions and craters in the terrain are marked as impassable (cells 7, 12, etc). There are inclined areas in the terrain shown in Orange cells. The motion dynamics of the rover is interpreted as a hybrid system. The rover's continuous motion and battery discharge dynamics can be different in each cell.
For instance, the motion and battery depletion dynamics in an inclined region is different from the dynamics in flat-regions and regions where soil and rock samples are collected. When a rover makes a transition from one cell to another, it starts to follow the dynamics of the new cell instantaneously. The discrete dynamics here capture the connectivity of the cells in the presence of mountains and craters where the rover cannot move. The rover's movement is restricted to one of its adjacent cells and movement to diagonal cells is prohibited.

\begin{figure}[htbp]
    \centering
    \includegraphics[width=0.66\textwidth]{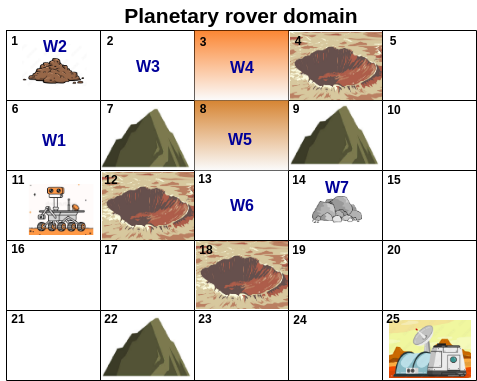}
    \caption{Rover-domain is depicted. Initially, the rover is at cell 11. Mountains and craters are impassable regions of the terrain such as cells 7, 12, etc. The rover needs to collect soil samples from cell 1 and rock samples from cell 14. The base station is at cell 25. The up-slope areas in the terrain are shown in Orange. The inevitable waypoints in the domain are cells $w1$ to $w7$ marked in Blue.}
    \label{fig:rover-domain}
\end{figure}

When a planner reports the planning task as unsolvable, our algorithm identifies ordered waypoints that ought to be reached in any plan in order to achieve the task. For instance, in the discussed domain, our algorithm detects that the cells 6-1-2-3-8-13-14 must be visited by any valid plan to solve the planning task. These are marked as ordered waypoints $w1 - w7$ in the figure. Our proposed algorithm envisions such waypoints as sub-goals of the planning problem. The unreachability of any of these waypoints under the domain dynamics is reported as an explanation to the unsolvability of the original planning problem. The computational challenge lies in finding the waypoints, finding an order between them and lastly, finding the earliest waypoint which is unreachable under the dynamics.  We propose a method of waypoint identification by casting the problem as an instance of the longest common subsequence problem, a widely popular problem in computer science, typically considered as an illustrative example for the dynamic programming paradigm. Unreachability of a waypoint is determined using a bounded model checker. For instance, given the initial rover battery charge of 10 units and the battery depletion rates of the cells (depletion rate of 1 unit in cells except in cell 1 and cell 14 where soil and rock sampling depletes battery at a higher rate of 2 units, and in the cells with inclination having a depletion rate of 3 units), reachability analysis determines that cell 13 is the first unreachable waypoint and reports this as an explanation of unsolvability.

Some of the waypoints as sub-goals may be \textbf{explicitly} known from the planning problem description itself. For example, the planetary rover domain has two sub-goals explicitly mentioned, namely the collection of soil and rock samples from cell 1 and cell 14 respectively, marked as waypoints $w2$ and $w7$ respectively. There may be sub-goals that are not apparent from the problem description explicitly but they are \textbf{implicitly} mandatory to complete the bigger planning task. We term these as inevitable waypoints. For example, the inevitable waypoints in our planetary rover domain are marked $w1$ and $w3-w6$ in the figure. Our waypoint detection algorithm detects both the explicit as well as implicit waypoints. In the following section, we formally describe the domain representation and the explanation problem we intend to solve.

%% file: problem-statement.tex
\section{Problem Overview} \label{problem-overview}

We begin with a formal definition of a planning problem in a hybrid system. A hybrid system exhibits an interplay of discrete and continuous dynamics. Hybrid automata are a well-known mathematical model for such systems \cite{ALUR19953,10.1007/3-540-57318-6_30}. 
We now define a hybrid automaton model of a hybrid system.

\begin{definition} \label{def:HA} \cite{ALUR19953}
A hybrid automaton (HA) is a seven-tuple HA=\big(\emph{Loc}, \emph{Var}, \emph{Flow}, \emph{Init}, \emph{Lab}, \emph{Edge}, \emph{Inv}\big) where:
\begin{itemize}
    \item \emph{Loc} is a finite set of vertices called locations.
    \item \emph{Var} is a finite set of real-valued variables. A valuation $v$ is an assignment of a real value to each variable $x \in$ \emph{Var}. We write $V$ for the set of all valuations.
    \item \emph{Flow} is a mapping from each location $l \in \emph{Loc}$ to a set of differential equations $\{\dot{x} \in f(x_1,\ldots,x_{|Var|}) \mid x \in \emph{Var}\}$, where $\dot{x}$ denotes the rate of change of variable $x$.
    \item \emph{Init} is a tuple $\langle l_{0}, S \rangle$ such that $l_{0} \in \emph{Loc}$ and $S \subseteq V$.
    \item \emph{Lab} is a finite set of labels.
    \item \emph{Edge} is a finite set of transitions $e = (l, a, g, r, l’)$, each consisting of a source location $l \in \emph{Loc}$, a target location $l'\in \emph{Loc}$, a label $a\in \emph{Lab}$, a guard $g \subseteq V$ and a reset map $r: \mathbb{R}^{|Var|} \to 2 ^ {\mathbb{R}^{|Var|}}$.
    \item \emph{Inv} is a mapping from each location $l\in \emph{Loc}$ to a subset of valuations $V$. $\Box$
\end{itemize}
\end{definition}

\noindent Figure~\ref{fig:rover-model} shows the hybrid automaton for the planetary rover domain, as an example. Each cell in Figure~\ref{fig:rover-domain} is represented as a location of the automaton with an invariant that is the region enclosed by the cell. The battery charge depletion rate and motion dynamics of the rover within the cell is modeled as flow equations of the location. The cell-to-cell movement of the rover is given as transitions between locations. The initial location is shown in green. The yellow locations represent the regions where the rover collects soil and rock samples. The orange locations represent the inclined regions. The base station, the rover's destination, is shown in red. An impassable location (loc7) is shown in grey with no incoming or outgoing edges. 

\begin{figure*}[htbp]
    \centering
    \includegraphics[width = 0.95 \textwidth]{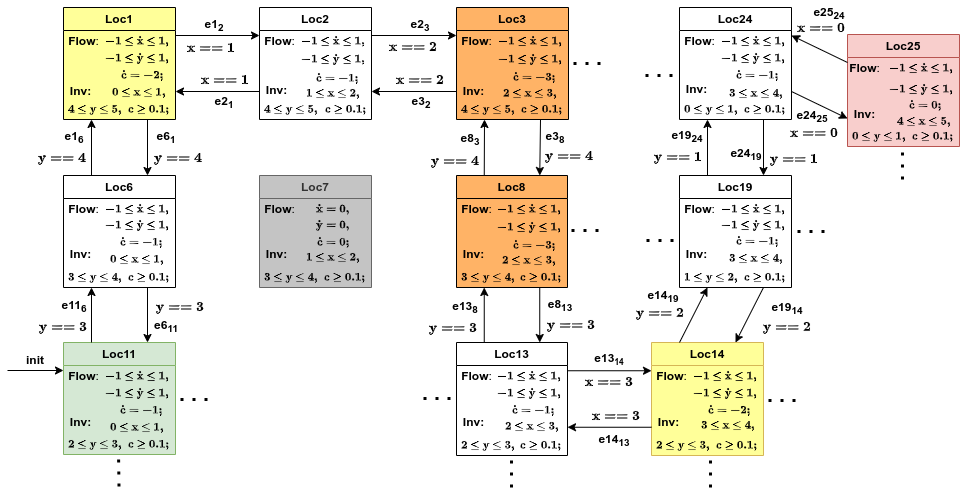}
 \caption{A hybrid automaton model of the planetary rover domain (partially shown).}
    \label{fig:rover-model}
\end{figure*}

\noindent A planning problem typically consists of a domain description and the initial and goal states of the planning task. In this work, as part of the problem description, we also consider a bound on how many times the actions of the domain can be applied. We define a planning problem for a hybrid system as follows:

\begin{definition} \label{def:planning_problem}
A \emph{planning problem} $\Pi$ for a hybrid system is a three-tuple $(Dom$, $Prob$, $Depth)$, where
\begin{itemize}
  \item $Dom$ \cite{DBLP:journals/tecs/SarwarRB23} is represented as a hybrid automaton HA. 
  \item $Prob$ is a tuple $\langle Init, Goal\rangle$ representing a problem description, where $Init$ and $Goal$ define the initial and the goal states. 
  \item \emph{Depth} defines the bound on how many actions can be applied in a plan. 
  $\Box$
\end{itemize}
\end{definition}

\noindent A state of a \emph{HA} is a pair $(l,v)$ consisting of a location $l \in$ Loc and a valuation $v\in $ Inv($l$). The Init of the automaton defines the set of initial states $\{(l_{0}, v) \mid v \in S\}$ of a planning problem. The $Goal$ states of the planning problem are given as a tuple $\langle l_{goal}, S_{goal} \rangle$ such that $l_{goal} \in Loc$ and $S_{goal} \subseteq$ Inv($l_{goal}$). The tuple represents goal states $\{(l_{goal}, v) \mid v \in S_{goal} \}$. The set of labels of the hybrid automaton corresponds to the available actions for a plan. A state can change either due to a transition in the automaton or due to the passage of time where the variables evolve according to the flow in a location. We refer to the former as \emph{discrete-transition} and the latter as \emph{timed-transition}. A discrete transition happens due to the application of an action given by the label of the transition provided the valuation of the state satisfies the guard of the transition. On taking a discrete transition, the valuation of the new state must follow the reset map of the transition. We now define a plan for a planning problem $\Pi$: 

\begin{definition}\label{def:plan}
A plan for a planning problem $\Pi$ is a tuple $\langle$ $\lambda_n$, $makespan$ $\rangle$ where $\lambda_n$ is a finite sequence of $n$ pairs $\langle t_i, a_i \rangle$. In the pair, $t_i \in \mathbb{R}^+$ is the time instance of executing the action $a_i \in$ \emph{Lab}. In the sequence, $t_i$ is non-decreasing. The $makespan$ is the duration of the plan. $\Box$
\end{definition}

\noindent We refer to the length of a plan to be the number of pairs $\langle t_i, a_i \rangle$ in $\lambda_n$. An \emph{executable} plan on a \emph{HA} is defined as follows:

\begin{definition}\label{def:run}
A plan $\langle \lambda_n, makespan \rangle$ for a planning problem $\Pi$ is called executable on the domain HA of $\Pi$ if and only if the application of the plan on HA results in an alternating sequence of timed and discrete transitions, called a run of the HA depicted as:
\begin{align*}
(l_{0},v_{0}) &\xrightarrow{\tau_0}(l_{0},v'_0) \xrightarrow{a_1} (l_1,v_{1}) \xrightarrow{\tau_1} (l_1,v'_{1})\xrightarrow{a_2} \ldots \xrightarrow{a_{n}}(l_{n},v_n) \xrightarrow{\tau_{n}} \\(l_{n},v'_{n})
\end{align*}
where 
(i) $(l_{0},v_{0}) \in \emph{Init}$ 
(ii) $a_i$ is a label of some edge $e_i \in \emph{Edge}$ such that $l_{i-1}$ is the source and $l_{i}$ is the destination location of $e_i$, $v'_{i-1} \in g$ and $v_{i} \in r(v_{i-1})$ where $g$ is the guard and $r$ is the reset map of $e_i$, $\forall i \in \llbracket 1..n\rrbracket$ 
(iii) The transitions labeled $\tau_i \in \mathbb{R}$ represent timed transitions with $\tau_i$ being the time of dwelling in the location, with the constraint that $\forall t \in [0, \tau_i]$, the timed transition $(l_i,v_{i}) \xrightarrow{t} (l_i,v^+_{i})$ has $v^+_{i} \in Inv(l_i)$, $\forall i \in \llbracket 0..n\rrbracket$. 
(iv) $\langle \sum_{j=0}^{i-1} {\tau_i}, a_i \rangle$ is a pair in $\lambda_n$, $\forall i \in \llbracket 1..n\rrbracket$ (v) $\sum_{i=0}^{n}{\tau_i}$ = \emph{makespan}.
$\Box$
\end{definition}
\noindent The length of a run is the number of discrete transitions it contains. In control theoretic terms, a plan is a control strategy that acts on a plant, a hybrid automaton in our context. The application of a control strategy on a plant results in a controlled execution of the plant, which we call a \emph{run} in our context. Due to uncertainties modeled in a \emph{HA}, an application of a plan may result in more than one runs. A plan is called \emph{valid} if it is executable, its applications on the \emph{HA} results in a run from a state in $Init$ to a state in $Goal$ and the length of the plan is less than or equal to Depth. In the following text, we write "a run of a valid plan" as a short-form of saying "the resulting run of the domain \emph{HA} of $\Pi$ due to the application of a valid plan". A planning problem $\Pi$ is called solvable if it admits a valid plan. If no such plan exists, then the planning problem is said to be unsolvable.
\begin{definition}
    A planning problem $\Pi$ is unsolvable if it admits no valid plan. $\Box$
\end{definition}
\noindent The problem addressed in this work is as follows:

\begin{prob}
Given an unsolvable planning problem $\Pi$, generate an artifact automatically that explains why is $\Pi$ unsolvable.
\end{prob}
\noindent In the subsequent sections, we describe the details of the explanation artifact and the algorithm to generate the same.

%% file: methodology.tex
\section{Methodology} \label{methodology}
\begin{figure*}[htbp]
    \centering
    \includegraphics[width= \textwidth]{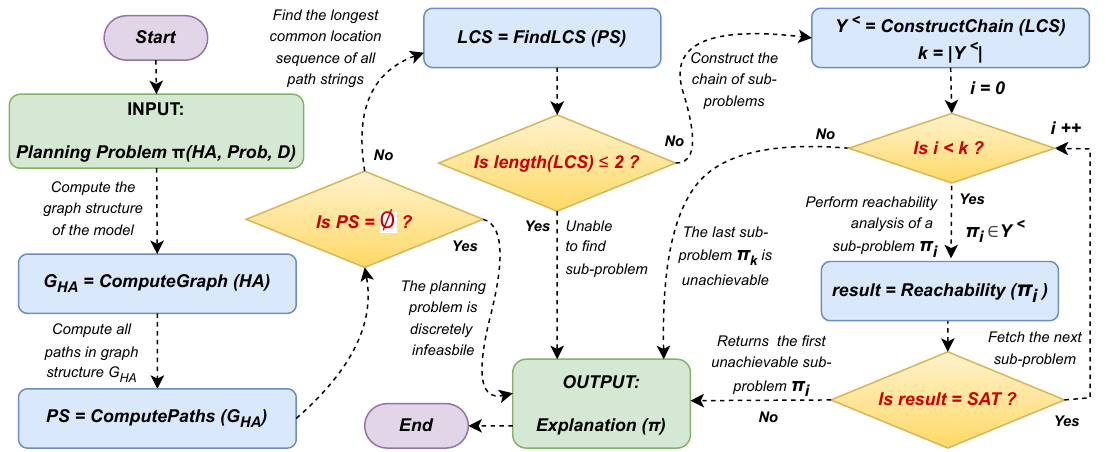}
    \caption{Proposed Explanation Framework.}
    \label{fig:framework}
\end{figure*}

\noindent In this section, we describe our explanation algorithm which takes an unsolvable planning problem $\Pi$ as input and computes an artifact $Explanation (\Pi)$, which we define later in the text (Defn. 7). Our explanation algorithm attempts to divide an unsolvable planning problem $\Pi$ into several sub-problems, following the common divide-and-conquer paradigm of problem-solving. These sub-problems have the property that each must be solvable for $\Pi$ to be solvable. Identifying these sub-problems is computationally challenging and is the key to generating the explanation artifact. The proposed algorithm takes a layered approach. The sub-problems are determined by taking into consideration only the discrete dynamics of the hybrid automaton. Once the sub-problems are identified, the explanation artifact is generated by considering the hybrid dynamics in its entirety. The abstraction of the continuous dynamics in the first phase allows us to work in the domain of graphs and consequently, we show a reduction from finding sub-problems to finding the longest common sub-sequence of a finite set of strings, a well-known problem in algorithms. Finally, the feasibility of the identified sub-problems is verified using symbolic reachability analysis and the explanation is generated which highlights which of the sub-problems is responsible for the unsolvability of $\Pi$. A schematic diagram of our explanation framework is shown in Figure~\ref{fig:framework}. We now present the details. 
%

\subsection{Decomposition into Sub-Problems}

We now define the notion of a sub-problem as a relation between hybrid planning problems.
\begin{definition}\label{def:subproblem}
    Given planning problems $\Pi_i$ = $(Dom_i$, $Prob_i$, $Depth_i)$ and $\Pi_j$ = $(Dom_j$, $Prob_j$, $Depth_j)$, we say $\Pi_i$ is a sub-problem of $\Pi_j$ when the following conditions hold: (a) $Dom_i = Dom_j$, (b) $Depth_i = Depth_j$, and (c) $Prob_i$, $Prob_j$ are tuples $\langle$ $Init_i$, $Goal_i$ $\rangle$ and $\langle$ $Init_j$, $Goal_j$ $\rangle$ resp. where \begin{itemize}
            \item $Init_i = Init_j$ and 
            \item Every run of a valid plan of $\Pi_j$ intersects with $Goal_i$. 
        \end{itemize}
\end{definition}
\noindent $\Pi_i$ and $\Pi_j$ only differ in the goal states in a way that every valid run of $\Pi_j$ goes past some goal state of $\Pi_i$. $\Pi_i$ is thus an inevitable \emph{waypoint} to $\Pi_j$.

\begin{exmp} \label{exp:sub-problem-example}
    Consider a planning problem in a rover-like domain explained above, the state-space shown in Figure~\ref{fig:subgoal-example}. The states within each cell belong to the invariant of a distinct location of the HA model with nine locations. The initial states are in the green region and the goal states are in the red region. The shaded regions are impassable. 
    Some of the runs of valid plans are shown as red trajectories from an initial state to a goal state. Note that any run of a valid plan has to pass through the blue region of the state-space and therefore when taken as new goal states, gives a sub-problem to the planning problem and is an inevitable waypoint.
\end{exmp}
\begin{figure}[htbp]
  \includegraphics[scale=0.50]{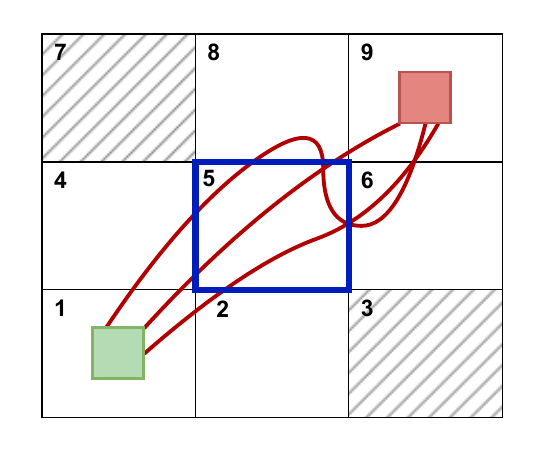}
  \caption{All runs of valid plans from initial to the goal must pass through the blue doorway. We envision the blue region as an inevitable way-point to the planning problem.}
  \label{fig:subgoal-example}
\end{figure}

\noindent  In the following text, we interchangeably refer to a sub-problem as a \emph{waypoint}. We write $\Pi_i < \Pi_j$ to say that $\Pi_i$ is a sub-problem of $\Pi_j$.  A planning problem $\Pi$ can have multiple sub-problems. We denote the set of all sub-problems of $\Pi$ by $\Pi^{\sqcup}$.
\begin{equation}
\Pi^{\sqcup} = \{ \Pi' \mid \Pi' < \Pi \}
\end{equation}

\subsubsection{Abstraction over Planning Problems}

\noindent The cardinality of $\Pi^{\sqcup}$ can be potentially infinite. We therefore present the following construction of a finite set of planning problems $\Pi^*$ for a given $\Pi$ induced by its \emph{HA} domain. We also define an abstraction function $\alpha$ that maps each $\Pi' \in \Pi^{\sqcup}$ to an element of $\Pi^*$ such that if $\Pi' < \Pi$ then $\alpha(\Pi') < \Pi$. We then proceed with our analysis on this finite abstraction $\Pi^*$.

\begin{definition}\label{abstractPi}
Let $\Pi = ( $HA$, \langle Init, Goal \rangle, Depth)$ be a planning problem. We define a finite set of planning problems $\Pi^*$ as follows:
\begin{equation}
    \Pi^* = \bigcup_{\ell \in Loc} \Pi_{\ell}
\end{equation}
where \emph{Loc} is the finite set of locations of HA, $\Pi_{\ell}$ = $\langle$ $HA$, $\langle Init$, $Goal_{\ell} \rangle$, $Depth$ $\rangle$ and $Goal_{\ell} = \langle \ell, Inv(\ell) \rangle$.
\end{definition}

\noindent The cardinality of $\Pi^*$ will be the cardinality of Loc of \emph{HA}. For example, the cardinality of $\Pi^*$ for the problem $\Pi$ of Figure \ref{fig:subgoal-example} is nine where each problem will have the tuple consisting of a location of the automaton and the corresponding invariant as its \emph{Goal}. The definition of $\alpha:\Pi^{\sqcup} \to \Pi^*$ for a given $\Pi' = (HA,\langle Init, Goal\rangle, Depth \rangle)$, where $Goal = \langle \ell_{goal}, S_{goal} \rangle$ is given as.   
\begin{equation}
\alpha(\Pi')=\Pi_{\ell_{goal}}
\end{equation}

\begin{proposition}
If $\Pi' \in \Pi^{\sqcup}$ then $\alpha(\Pi') \in \Pi^{\sqcup}$.
\end{proposition}
\begin{proof}
The goal states of $\Pi'$ is a subset of the goal states of $\alpha(\Pi')$ and therefore if every valid run of $\Pi$ intersect the goal states of $\Pi'$, then they also intersect the goal states of $\alpha(\Pi')$. Hence, $\alpha(\Pi') < \Pi$ and thus $\alpha(\Pi') \in \Pi^{\sqcup}$. 
\end{proof} 

\begin{proposition}
The ordered pair $\langle \Pi^*, < \rangle$ is a partially ordered set (poset).
\end{proposition}

\begin{proof}
For $\langle \Pi^*, < \rangle$ to be poset, the binary relation $<$ on $\Pi^*$ should be reflexive, anti-symmetric, and transitive, that is, $<$ must be a partial order relation. From definition \ref{def:subproblem}, it is easy to see that every planning problem in $\Pi^*$ is a sub-problem to itself and hence reflexive. Transitivity and anti-symmetry also follow from the definition of sub-problem.
\end{proof}
\noindent \textbf{Chains} Recall that a subset $Y \subseteq \Pi^*$ of a partially ordered set $\langle \Pi^*, < \rangle$ is a \emph{chain} if $\forall \Pi_i, \Pi_j \in \Pi^* : (\Pi_i < \Pi_j) \lor (\Pi_j < \Pi_i)$. The length of a chain is the number of elements it has. A poset can have more than one chain of longest length.  

\begin{definition}
A chain of sub-problems of $\Pi$ is denoted by the set $Y^{<} \subseteq \Pi^*$ where every element of $Y^{<}$ is a sub-problem of $\Pi$ and $Y^{<}$ is a chain of $\langle \Pi^*, < \rangle$.
\end{definition}

\noindent We now describe the explanation artifact that we intend to generate with our explanation algorithm:

\begin{definition}
    Given an unsolvable planning problem $\Pi$,
    an explanation artifact $Explanation(\Pi)$ is a planning problem $\Pi_i \in \Pi^*$ such that:
    \begin{enumerate}
        \item[(i)] $\Pi_i \in Y^{<}$ for some $Y^{<}$ such that $\Pi_i$ is unsolvable.

        \item[(ii)] $\forall \Pi_j \in Y^{<}$ such that $\Pi_j < \Pi_i$ and $\Pi_j \neq \Pi_i$, $\Pi_j$ is solvable.
    \end{enumerate} 
\end{definition}

\noindent The explanation is thus the first unsolvable sub-problem of $\Pi$ in a chain of sub-problems in the poset $\langle \Pi^*, < \rangle$. As an illustration, assume that $Y^{<} = \{\Pi_1$, $\Pi_2$, $ \ldots$, $\Pi_n\}$ is a chain of sub-problems of $\Pi$ in $\langle \Pi^*,< \rangle$ having a total order as $\Pi_1 < \Pi_2 < \ldots < \Pi_n$, Explanation($\Pi$) is the $\Pi_i \in Y^<$ that is unsolvable where $\forall \Pi_j \in Y^<$ such that $\Pi_j < \Pi_i$, $\Pi_j$ is solvable. For instance, consider the planning problem of the motivating example. Our abstraction will render 25 planning problems in $\Pi^*$, each having one of the cell as the goal. An example of a chain of sub-problems is the chain consisting of 7 sub-problems shown as waypoints $w1-w7$ with a total order $w1 < w2 < \ldots < w7$  amongst them. Observe that this chain is also the longest possible chain of sub-problems in the poset. In this chain, the explanation generated will be the earliest $w_i$ that is unsolvable.

\noindent The goal of generating the proposed explanation artifact is to assist a human expert/control engineer in diagnosing the causes of unsolvability by \textbf{Localizing the earliest cause of unsolvability}. The detection of the earliest waypoint which is infeasible localizes the primitive cause of unsolvability in that sense. The intuition behind finding a chain of sub-problems is to have a causal analysis of the unsolvability of the planning problem.

\noindent We now present the algorithm to find Explanation($\Pi$) in the following section. We first show a reduction from finding a chain of sub-problems to finding a longest-common-subsequence of finitely many strings.

\subsection{Reduction to Longest Common Subsequence (LCS) Problem}
The computation of Explanation($\Pi$) first requires finding a chain of sub-problems of $\Pi$ in $ \langle \Pi^*, <\rangle$. We now show a reduction of this problem to the problem of finding a LCS of finitely many strings. Recall that \emph{reduction} is a way of converting one problem into another problem such that the solution of the second problem can be used to solve the first problem.  To present the reduction to LCS, we need the following definition of the graph of a hybrid automaton.

\begin{definition}
The graph of a hybrid automaton HA
is defined as $\mathcal{G_{HA}} = (V$, $E)$, where $V$ = \emph{Loc} and $E \subseteq \emph{Loc} \times \emph{Loc}$ such that for every $(l, a, g, r, l') \in \emph{Edge}$, there is an edge $(l, l') \in E$. $\Box$
\end{definition}

\begin{definition} \label{def:path}
    A path $p$ between locations $l_0$ and $l_n$ in $\mathcal{G_{HA}}$ is a sequence of $locations$ and $edges$ given as:
    \begin{align*}
    l_0 \xrightarrow{e_0} l_1 \xrightarrow{e_1} l_2 \xrightarrow{e_2} \ldots \xrightarrow{e_{n-1}} l_{n}
    \end{align*}
    where $l_i$ $\in$ $Loc$, $e_i$ $\in$ $Edge$, and $l_{i}$, $l_{i+1}$ are the source and destination of $e_i$ respectively. The length of a path is the number of edges it contains. 
    $\Box$
\end{definition}

\noindent We represent a path in a graph as a string $ps$ of the location sequence while eliminating the edges. For example, a path $l_0 \xrightarrow{e_0} l_1 \xrightarrow{e_1} l_2 \xrightarrow{e_2} \ldots \xrightarrow{e_{n-1}} l_{n}$ 
is represented as a string "$l_0 l_1 l_2 \ldots l_n$". Now, for the given unsolvable problem $\Pi$ =  (\emph{HA}, $\langle Init, Goal \rangle, Depth)$, we can compute all paths of length less than or equal to $Depth$ between $l_{0}$ and $l_{goal}$, the initial and the goal location in $Init$ and $Goal$ respectively. Since we are interested in paths of bounded length, there will be finitely many such paths. The string representations of all such paths is denoted by the set $PS(\Pi)$. 

The graph of the hybrid automaton provides a higher abstraction of the domain in the sense that if there is no path from $l_0$ to $l_{goal}$ in $\mathcal{G_{HA}}$, then there cannot be any valid run of a plan from $Init$ to $Goal$ and hence the planning problem is unsolvable. We may then identify the cause of unsolvability to be in the discrete dynamics, oblivious to the continuous dynamics of the domain. More importantly, as we shall see now, a chain of sub-problems can be identified from the longest common subsequence of the strings in $PS(\Pi)$. Finding a \emph{longest common subsequence} (LCS) between strings is a classic computer science problem. An LCS measures the closeness of two or more strings by finding the maximum number of identical symbols in them in the same order \cite{DBLP:conf/spire/BergrothHR00,DBLP:journals/jacm/Maier78,princeton1974finding}. Recall that a subsequence is different from a substring which additionally requires that the common symbols present in the strings are without gaps.  We now present the main result of the paper.

\begin{proposition}
Given a planning problem $\Pi$, computing a chain of sub-problems $Y^<$ in the poset $\langle \Pi^*, < \rangle$ can be reduced to computing a longest common subsequence of $PS(\Pi)$.     
\end{proposition}

\begin{figure*}[htbp]
    \centering
    \includegraphics[width= 0.6\textwidth]{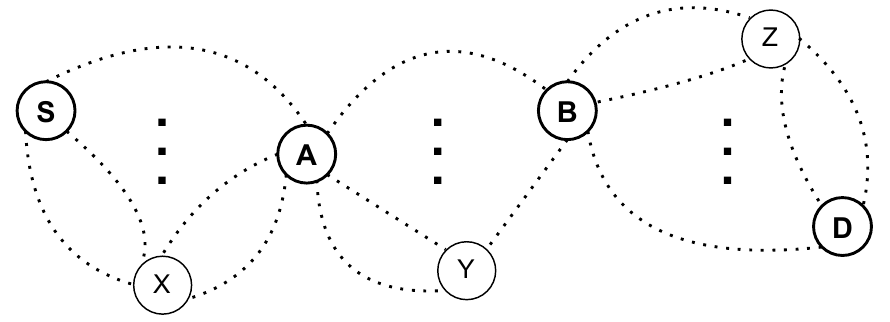}
    \caption{A depiction of a graph of a \emph{HA}. Let S and D be the source and the goal locations in a planning problem $\Pi$. Every path from S to D visits locations A and B in sequence. Thus, S-A-B-D is the LCS of paths in PS($\Pi$). Clearly, if a valid run of a plan exists in the \emph{HA}, the run must visit the invariants of S, A, B, and D sequentially.}
    \label{fig:chain-example}
\end{figure*}

\begin{proof}
$\mathcal{G_{HA}}$ = ($V$, $E$) is implicitly present in \emph{HA} of $\Pi$. Any graph search algorithm such as breadth-first search can compute paths in $PS(\Pi)$. Let  "$l_0 l_i l_j \ldots l_n$" be a LCS of path strings in $PS(\Pi)$. Being a common subsequence, every path from $l_0$ to $l_n$ visits these nodes in sequence. This implies that every valid run of $\Pi$ must visit the invariant of these locations in sequence, that is: $\Pi_{l_0} < \Pi_{l_i} < \Pi_{l_j} < \ldots < \Pi_{l_n}$. $\therefore Y^< = \{ \Pi_{l_0}, \Pi_{l_i}, \Pi_{l_j}, \ldots , \Pi_{n} \}$ is a chain in $\langle \Pi^*, < \rangle$. Figure \ref{fig:chain-example} shows a sketch of the proof idea.


\end{proof}



\begin{proposition}
Given a planning problem $\Pi$, the length of LCS of $PS(\Pi)$ is bounded by $|\Pi^*|$, the number of locations in HA of $\Pi$. \label{prop:lcs-length}
\end{proposition}
\begin{proof}
Every path in $PS(\Pi)$ longer than $|\Pi^*|$ will have cycles. The corresponding cycle free path will also be in $PS(\Pi)$ of length bounded by $|\Pi^*|$. Since the LCS cannot be longer than the shortest string in $PS(\Pi)$, its length is bounded by $|\Pi^*|$.  
\end{proof}

\noindent \textbf{Discussion:} If $PS(\Pi)$ is empty, our explanation algorithm terminates and reports that the planning problem is unsolvable due to the discrete dynamics, since there is no path between the initial and the goal locations. 

Observe that there is a connection between \emph{articulation points} or \emph{cut-vertices} of $\mathcal{G_{HA}}$ and the locations in an LCS of $PS(\Pi)$. An articulation point is a vertex of a graph removal of which along with its incident edges results in increasing the connected components in the graph. One can argue that every articulation point of $\mathcal{G_{HA}}$ whose removal results in distinct components such that one contains $l_0$ and the other other contains $l_{goal}$, will be a member of the LCS. This is because, every path from $l_0$ to $l_{goal}$ must contain such articulation points and therefore will be captured in the LCS. Let us call such articulation points as \emph{disconnecting articulation points} in the sense that their removal disconnects $l_0$ and $l_{goal}$. Note that every vertex in an LCS need not be such an articulation point of $\mathcal{G_{HA}}$. This is because a LCS is computed over paths in $PS(\Pi)$ which contains only paths of length bounded by a depth specified in the problem instance. There may be paths in $\mathcal{G_{HA}}$ of longer length which does not pass through one or more vertices in the LCS. Such vertices in the LCS cannot be disconnecting articulation points whose removal disconnects $l_0$ and $l_{goal}$. Consequently, if our algorithm finds a trivial LCS string of length two, which is $l_0-l_{goal}$, then $\mathcal{G_{HA}}$ has no disconnecting articulation point. Such an LCS is trivial because any valid plan of course must visit the invariant of $l_0$ followed by the invariant of $l_{goal}$. Therefore, Y$^{k}$ = $\{\Pi_{l_0}, \Pi_{l_{goal}} \}$ is always a chain of inevitable sub-problems for any planning problem $\Pi$. Although, finding a trivial LCS does not mean there are no articulation points in $\mathcal{G_{HA}}$. All articulations points in $\mathcal{G_{HA}}$ can be computed in polynomial time, but computing Explanation($\Pi$) additionally requires finding \emph{disconnecting articulation points} and an ordering on them based on the sub-problem relation. Thus, the polynomial time algorithm does not suffice.   


\noindent LCS of strings can be computed using the standard dynamic programming paradigm \cite{DBLP:conf/spire/BergrothHR00}. For instance, in the motivating example of planetary rover domain, the longest common sequence for the strings of paths of length bounded by 10 from $l_{11}$ to $l_{25}$ is $l_{11}$-$l_6$-$l_1$-$l_2$-$l_3$-$l_8$-$l_{13}$-$l_{14}$-$l_{25}$. In the next section, we show the computation of Explanation($\Pi$).

\subsection{Explanation Generation by Reachability Analysis}
The computed LCS is converted to sub-problems from the locations in the LCS. For each location $l$ in the LCS, we construct a sub-problem $\Pi_{l}$ (recall definition \ref{abstractPi}). Thus, we have the chain $Y^< = \{ \Pi_{l_0}, \Pi_{l_i}, \Pi_{l_j}, \ldots , \Pi_{l_n} \}$ for an LCS say "$l_0 l_i l_j \ldots l_n$". We verify the solvability of the sub-problems in $Y^<$ (inevitable waypoints) using bounded reachability analysis on the hybrid automaton domain by reintroducing the continuous dynamics (by reintroducing the location invariants, flow, transition guards, and resets). Given a hybrid automaton, a set of initial and goal states, and a bound of analysis say $d$, bounded reachability analysis is the method of computationally deciding whether any goal state is reachable from any initial state by a run of the automaton of length bounded by $d$. Therefore, bounded reachability of the goal states from initial states of a planning problem bounded by $Depth$, implies the existence of a run of a valid plan, which in turn implies the solvability of the given planning problem. In contrary, unreachability of the goal states implies the non-existence of any valid run and hence absence of a plan. As our planning problem under analysis is unsolvable, one or more of the waypoints in $Y^<$ must be unreachable. The reachability analysis of the waypoints is performed in the order in which they appear in the chain $Y^<$. If the sub-problem $\Pi_i$ is found to be  reachable, we proceed to check the reachability of the next sub-problem $\Pi_{i+1}$ in the chain. The first $\Pi_i$ which is unreachable is returned as $Explanation(\Pi)$, implying that it is the first unreachable sub-problem/waypoint in a chain of sub-problems/waypoints.



We use a bounded reachability analysis tool \textsc{Bach} \cite{DBLP:conf/fmcad/BuLWL08} for reachability analysis of the waypoints. \textsc{Bach} can analyse linear hybrid automata (LHA) \cite{DBLP:conf/lics/Henzinger96,DBLP:journals/entcs/LiAB07} and reports a reachability problem instance as \emph{satisfiable} when a run exists from an initial state to a goal state in the corresponding hybrid automaton for the planning problem of length bounded by a given depth, deciding the solvability of the corresponding planning problem. Otherwise, \textsc{Bach} reports the instance as \emph{unsatisfiable}, when no such run exists. Algorithm~\ref{algo:reach} takes a chain of sub-problems $Y^<$ as input. 
It returns the first unreachable sub-problem in the chain as an explanation of unsolvability. 

\begin{algorithm}[]
\SetKwInOut{Input}{input}\SetKwInOut{Output}{output}
\Input{A chain of sub-problems $Y^<$ = $\langle \Pi_{l_0}, \Pi_{l_i}, \Pi_{l_j}, \ldots , \Pi_{l_n} \rangle$}
\Output{First unreachable sub-problem $\Pi_{l_k}$ in the chain.}
\DontPrintSemicolon
\caption{\textsf{Generating Explanation($\Pi$) using Reachability Analysis}}
\label{algo:reach}
\BlankLine
\LinesNumbered
  {$n \gets$ $length(Y^<)$ \small{\tcc*{No. of sub-problems in $Y^<$}} }
  \For{$k \gets 0$ to $n$} {
     {$\mathcal{R} \gets$ \textsc{reachabilityProblem}(\emph{HA}, $Init$, $Goal_{l_k}$, $Depth$), where $Goal_{l_k}$ is the goal of $\Pi_{l_k}$ \small{\tcc*{Reduced to reachability problem}} }
     {$result \gets$ \textsc{BACH} ($\mathcal{R}$) \small{\tcc*{Use BACH}} }
     \eIf{$result$ == $\textit{SAT}$}
        {\textbf{continue} \small{\tcc*{When sub-problem $\Pi_{l_k}$ is reachable.}}}
        {\KwRet{$\Pi_{l_k}$} \small{\tcc*{Returns first unreachable sub-problem as the explanation.}}}
  }
\end{algorithm}





\subsection{Complexity Analysis}

In a \emph{HA} with $n$ locations and given a source and destination location, the worst-case complexity of computing all paths in $\mathcal{G}_{HA}$ of length at most $d$ from the given source to the destination is $O\big(n^d\big)$. 
The complexity of computing $LCS$ of strings corresponding to the paths (say $m$ many) of maximum length $d$ using standard dynamic programming is $O\big(d^m\big)$. It is known to be an NP-Hard problem \cite{DBLP:journals/asc/DjukanovicRB20,DBLP:journals/jacm/Maier78}. Although finding the inevitable way-points with the proposed algorithm turns out to be inefficient asymptotically, for problem instances of small size (\emph{HA} with a few locations and for a small depth d), we show empirically that our algorithm can generate inevitable way-points and the explanation artifact efficiently.



\paragraph{Complexity of Reachability Analysis :} The model checker \textsc{Bach} performs a path-oriented reachability analysis of a linear hybrid automaton. An initial to goal location path is encoded into a set of linear constraints, and consequently solved using linear programming problem solver. Details of the path encoding can be found in \cite{DBLP:conf/fmcad/BuLWL08,DBLP:conf/memocode/SarwarRB23}. Though practical LP solvers use variants of the Simplex algorithm, which runs efficiently on most practical problems, it is not a polynomial time algorithm in general. In theory, linear programming has been shown to be in class P \cite{khachiyan1979polynomial,KHACHIYAN198053}. Our algorithm calls \textsc{Bach} for each sub-problem in the computed chain. As the length of the computed chain, the LCS, is bounded by the number of locations of \emph{HA} (proposition \ref{prop:lcs-length}), the complexity is $|\Pi^*|$ times the complexity of LP solving. 

%% file: evaluation.tex
\section{Results and Implementation} \label{eval}
In this section, we present the performance of our framework on several unsolvable hybrid planning problem instances.

\subsection{Experimental Setup}

\paragraph{Benchmarks}: A brief description of the planning domains is given as follows: 
\textbf{Planetary rover domain} is presented in Section~\ref{motive} and a pictorial overview is shown in Figure~\ref{fig:rover-domain}.
The planning task for the rover is to reach the base station from its initial location after collecting soil and rock samples from the designated sites. \textbf{City route-network domain} presents a route-network of a city where important places are given as junctions in the network. The planning problem for a battery-powered car is to navigate through the city's route network to reach its destination. \textbf{Warehouse automation domain}~\cite{DBLP:conf/memocode/SarwarRB23} represents a scenario where a robot operates to manage the inventories of a warehouse. The floor map of the warehouse is given as grid cells. Few cells in the warehouse are blocked, whereas, on a few cells, the robot depletes more energy due to the condition of the surface such as a oil spillage or being bumpy. The planning problem is for the robot to carry a consignment from its initial location to a goal location. We have crafted warehouse scenarios of varying grid dimensions for evaluating our algorithm. \textbf{Water-level monitor}~\cite{DBLP:conf/fmcad/BuLWL08} represents a system that controls the water level in a reservoir. The system goes into an unsafe state if the water level in the reservoir meets underflow or overflow conditions. The planning task is to drive the system to an unsafe state from a given initial state. \textbf{NAV}~\cite{ARCH-COMP24:ARCH_COMP24_Category_Report} models the motion of a point robot in a $2$-dimensional plane, partitioned into $3^2$ rectangular regions, and each such region is associated with a vector field described by the flow equations. The planning problem is to find a trajectory from an initial state to a goal state. \textbf{NRS}~\cite{F-Wang-NRS, ARCH-COMP24:ARCH_COMP24_Category_Report} represents a nuclear reactor system consisting of 2 rods that absorb neutrons from heavy water when inserted, and a controller that schedules the insertion of the rods into the heavy water. 
The system is considered safe if there is exactly one rod absorbing neutrons in the heavy water at any instant of time. The planning problem is to find an unsafe execution of the system from a given initial state. A detailed description of the city route and warehouse automation domains is given in the Appendix. 

We constructed the \emph{Planetary rover} and the \emph{City route-network} domains for evaluating our algorithm. The \emph{Warehouse automation} domain is taken from \cite{DBLP:conf/memocode/SarwarRB23}.
The rest of the domains are from verification problem instances in linear hybrid systems. For instance, the \emph{Water-level monitor} domain is a benchmark taken from \cite{DBLP:conf/fmcad/BuLWL08}, whereas \emph{NAV} and \emph{NRS} are benchmarks taken from \textsc{Arch-comp 24 pcdb} category ~\cite{ARCH-COMP24:ARCH_COMP24_Category_Report}.
In these domains, we pose the safety property verification problem as planning problem instances. The planning problems taken for evaluation are all known to be unsolvable.


\paragraph{Implementation}: All experiments are performed on a machine with 8 GB RAM, Intel Core i5-8250U@1.60GHz, and 8 core processor with Ubuntu 18.04 64-bit OS. All benchmark domains, the problem files, and the code base can be found at: \url{https://gitlab.com/Sazwar/Sub-goal-Construction}.

\subsection{Evaluation}

\begin{table*}[]
  \centering\tiny
  \resizebox{\textwidth}{!}{
    \begin{tabular}{|c|c|c|c|c|c|c|c|c|c|c|}
    \hline
         \multicolumn{2}{|c|}{\multirow{2}{*}{Benchmarks}} & \multirow{2}{*}{\#Locs} & \multirow{2}{*}{\#Trans} & \multirow{2}{*}{Depth} & \multirow{2}{*}{$\mid$PS($\Pi$)$\mid$}  & \multirow{2}{*}{$\lvert\text{Y}^<\rvert$}   & \#Feas. & \multirow{2}{*}{Exp($\Pi$)} & Time & Memory \\
         \multicolumn{2}{|c|}{} & &  & & & & wps & & (in sec) & (in MB) \\

    \hline
           \multicolumn{2}{|c|}{Planetary} & \multirow{2}{*}{25} & \multirow{2}{*}{40} & 12 & 3 & \multirow{2}{*}{9} & \multirow{2}{*}{6} & \multirow{2}{*}{Loc13} & 0.50 & 7.5 \\\cline{5-6}\cline{10-11}
           \multicolumn{2}{|c|}{rover (PR)} & & & 20 & 34361 & & & & 3.81 & 471.6 \\
    \hline 
           \multicolumn{2}{|c|}{City} & \multirow{2}{*}{10} & \multirow{2}{*}{25} & 10 & 468 & \multirow{2}{*}{4} & \multirow{2}{*}{2} & \multirow{2}{*}{Loc7} & 0.73 & 10.1 \\\cline{5-6}\cline{10-11}
           \multicolumn{2}{|c|}{route (CR)} & & & 15 & 92172 & & & & 5.93 & 1075.8 \\
         
    \hline 
           & \multirow{2}{*}{6x4} & \multirow{2}{*}{24} & \multirow{2}{*}{50} & 10 & 36 & \multirow{2}{*}{6} & \multirow{2}{*}{4} & \multirow{2}{*}{Loc17} & 1.57 & 8.8 \\\cline{5-6}\cline{10-11}
           & & & & 15 & 40998 & & & & 4.17 & 507.5 \\\cline{2-11}
           
           \multirow{4}{*}{Warehouse} & \multirow{2}{*}{6x6} & \multirow{2}{*}{36} & \multirow{2}{*}{78} & 12 & 12 & 8 & \multirow{2}{*}{4} & \multirow{2}{*}{Loc28} & 0.54 & 9.2 \\\cline{5-7}\cline{10-11}
           & & & & 17 & 5816 & 6 &  & & 4.47 & 884.9 \\\cline{2-11}
           
           \multirow{2}{*}{automation (WA)} & \multirow{2}{*}{8x8} & \multirow{2}{*}{64} & \multirow{2}{*}{100} & 12 & 16 & 9 & 7 & \multirow{2}{*}{Loc41} & 1.37 & 17.5 \\\cline{5-8}\cline{10-11}
           & & & & 17 & 10214 & 3 & 1 & & 8.62 & 1454.2 \\\cline{2-11}

           & \multirow{2}{*}{10x10} & \multirow{2}{*}{100} & \multirow{2}{*}{178} & 12 & 2 & 11 & 7 & \multirow{2}{*}{Loc57} & 4.13 & 133.9 \\\cline{5-8}\cline{10-11}
           & & & & 15 & 78 & 8 & 4 & & 25.14 & 1445.3 \\
    \hline
           \multicolumn{2}{|c|}{Water-level} & \multirow{2}{*}{6} & \multirow{2}{*}{6} & 20 & 5 & \multirow{2}{*}{3} & \multirow{2}{*}{2} & \multirow{2}{*}{Loc6} & 0.05 & 5.7 \\\cline{5-6}\cline{10-11}
           \multicolumn{2}{|c|}{monitor (WLM)} & & & 50 & 12 & & & & 0.05 & 5.7 \\
    \hline
          \multicolumn{2}{|c|}{\multirow{2}{*}{NAV}} & \multirow{2}{*}{9} & \multirow{2}{*}{24} & 10 & 2325 & \multirow{2}{*}{2} & \multirow{2}{*}{1} & \multirow{2}{*}{Loc6} & 0.39 & 9.5 \\\cline{5-6}\cline{10-11}
          \multicolumn{2}{|c|}{} & & & 15 & 149733 & & & & 3.41 & 773.2 \\
    \hline   
          \multicolumn{2}{|c|}{\multirow{2}{*}{NRS}} & \multirow{2}{*}{27} & \multirow{2}{*}{30} & 15 & 312 & \multirow{2}{*}{2} & \multirow{2}{*}{1} & \multirow{2}{*}{Loc25} & 0.03 & 7.6 \\\cline{5-6}\cline{10-11}
          \multicolumn{2}{|c|}{} & & & 20 & 7812 &  &  & & 0.14 & 22.1 \\
    \hline
    
    \end{tabular}
  }
  \caption{Explanation generation on unsolvable planning problem instances.}
  \label{table:table1}
\end{table*}

Table~\ref{table:table1} shows results of our framework on several hybrid systems benchmark domains. Each of the domains is presented with an unsolvable planning problem instance with varying bounds on the plan depth to test the scalability of the framework. \emph{Benchmark} represents the planning domains together with the planning problem,
while \emph{\#Loc} and \emph{\#Trans} report the number of locations and edges in the hybrid automaton of the domain respectively, showing the size of the domain.
\emph{Depth} presents the bound on the plan length.
\emph{$\mid$PS($\Pi$)$\mid$} specifies the number of path strings corresponding to the paths from the initial to the goal location of the planning problem instance in the abstract graph structure of the HA domain. Recall that we look at all these paths while computing the longest common location sub-sequence which gives us the inevitable waypoints in the chain of sub-problems in Y$^<$.
\emph{$\mid$Y$^<$$\mid$} gives us the chain length, which emphasizes the number of inevitable sub-problems detected by our framework, and \emph{\#Feas. wps} denotes the number of solvable sub-problems/waypoint(s) in the chain.
\emph{Exp($\Pi$)} presents the first unsolvable sub-problem in Y$^<$, and thereby, the first infeasible waypoint for the planning problem. A location $loc$ in the Exp($\Pi$) column represents the first sub-problem $\Pi_{loc}$ in the chain Y$^<$ that is unsolvable. For example, \emph{Loc13} corresponding to the entry of \emph{Planetary rover domain} reports the sub-problem $\Pi_{loc13}$ as the explanation of unsolvability of the problem instance. \emph{Time} and \emph{Memory} report the corresponding execution time and memory usage incurred by our framework.

\begin{table}[]
  \centering
    \begin{tabular}{|c|c|c|c|c|c|c|}
    \hline
        \multicolumn{2}{|c|}{\multirow{2}{*}{Benchmarks}} & \multirow{2}{*}{Depth} & \multicolumn{3}{c|}{Time (in secs)}  & \multirow{2}{*}{AT}\\\cline{4-6}
        \multicolumn{2}{|c|}{} & & (a) PE & (b) Finding Y$^<$ & (c) RA & \\

    \hline
           \multicolumn{2}{|c|}{\multirow{2}{*}{PR}} & 12 & 0.01 & 0.01 & 0.48 & 0.50 \\\cline{3-7}
           \multicolumn{2}{|c|}{} & 20 & 2.20 & 1.01 & 0.60 & 3.81\\
    \hline 
           \multicolumn{2}{|c|}{\multirow{2}{*}{CR}} & 10 & 0.03 & 0.01 & 0.69 & 0.73 \\\cline{3-7}
           \multicolumn{2}{|c|}{} & 15 & 4.72 & 0.23 & 0.98 & 5.93 \\
    \hline 
           \multirow{8}{*}{WA} & \multirow{2}{*}{6x4} & 10 & 0.01 & 0.01 & 1.54 & 1.57 \\\cline{3-7}
           & & 15 & 2.29 & 0.16 & 1.72 & 4.17 \\\cline{2-7}
           
           & \multirow{2}{*}{6x6} & 12 & 0.01 & 0.01 & 0.51 & 0.54 \\\cline{3-7}
           & & 17 & 3.67 & 0.04 & 0.76 & 4.47 \\\cline{2-7}
           
           & \multirow{2}{*}{8x8} & 12 & 0.05 & 0.01 & 1.31 & 1.37 \\\cline{3-7}
           & & 17 & 7.53 & 0.03 & 1.06 & 8.62 \\\cline{2-7}
           
           & \multirow{2}{*}{10x10} & 12 & 0.52 & 0.02 & 3.59 & 4.13 \\\cline{3-7}
           & & 15 & 21.68 & 0.02 & 3.44 & 25.14 \\
    \hline
           \multicolumn{2}{|c|}{\multirow{2}{*}{WLM}} & 20 & 0.01 & 0.01 & 0.03 & 0.05 \\\cline{3-7}
           \multicolumn{2}{|c|}{} & 50 & 0.01 & 0.01 & 0.03 & 0.05 \\
    \hline
           \multicolumn{2}{|c|}{\multirow{2}{*}{NAV}} & 10 & 0.01 & 0.01 & 0.37 & 0.39 \\\cline{3-7}
            \multicolumn{2}{|c|}{} & 15 & 2.78 & 0.21 & 0.42 & 3.41 \\
    \hline   
           \multicolumn{2}{|c|}{\multirow{2}{*}{NRS}} & 15 & 0.01 & 0.01 & 0.01 & 0.03 \\\cline{3-7}
            \multicolumn{2}{|c|}{} & 20 & 0.12 & 0.01 & 0.01 & 0.14 \\
    \hline
    
    \end{tabular}
  \caption{Performance analysis of our framework. \emph{PE} indicates Path-exploration time, \emph{RA} indicates Reachability analysis time, and \emph{AT} indicates Accumulative Time = (a) + (b) + (c).}
  \label{table:table2}
\end{table}
Table~\ref{table:table2} presents a detailed diagnosis of the execution time taken for explanation generation, showing the time taken for computing all initial to goal paths (PS), Computing a chain of inevitable sub-problems (Y$^<$), and reachability analysis to find the first unsolvable planning problem in the chain (Explanation($\Pi$)).

\subsection{Analysis of results}
Table~\ref{table:table1} shows that our framework identifies a chain of inevitable waypoints and an explanation of unsolvability efficiently. Performance degrades with an increase in the depth bound of the planning problem instance. This is clearly because increasing depth results in an exponential increase in the number of initial to goal paths which also increases the time to compute LCS of path strings. In NAV and NRS, our algorithm reports the trivial chain of waypoints which is visiting the initial location followed by visiting the goal location as inevitable. Note that this is because the graph of the these domain does not have any disconnecting articulation point (refer to the discussion section). 
Memory usage exceeds 500 MB in a few instances. 
This is because of a \emph{bfs} (breadth-first search) based path exploration where the size of the bfs queue increases exponentially at each level due to branching factor. Table~\ref{table:table2} shows the performance of the three major components of the algorithm. 
The path-exploration time and reachability analysis by the bounded model checker dominates the overall time taken by the algorithm. The results emphasize that it can quickly identify the sub-problems for a planning problem.
\begin{figure*}[htbp]
    \centering
    \begin{subfigure}[b]{0.48\textwidth}
        \centering
        \includegraphics[scale=0.4]{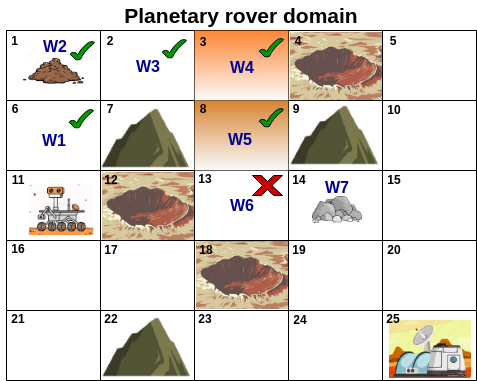}
        \caption{$W6$ is the first unreachable waypoint and serves as an explanation of the planning problem being unsolvable.}
        \label{fig:waypoints}
    \end{subfigure}
    \hfill
    \begin{subfigure}[b]{0.48\textwidth}
          \centering
         \includegraphics[height=165pt,width=0.97\textwidth]{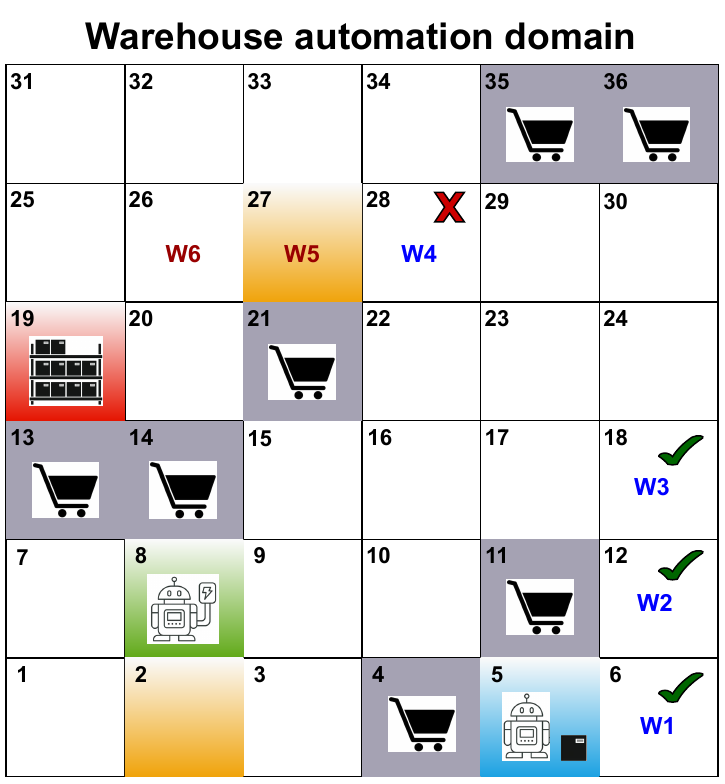}
         \caption{
         $W4$ is the first unreachable waypoint for the planning problem.
         }
         \label{fig:waypoints1}
    \end{subfigure}
    \caption{Illustration of results in the example scenarios.}
    \label{fig:res-way}
\end{figure*}


In Figure~\ref{fig:res-way}, we illustrate our method in the example scenarios of \emph{Planetry rover} and \emph{Warehouse automation} domains.
In the motivating example problem instance in \emph{Planetary rover} domain, the algorithm identified 7 sub-problems $\{W1, W2,\ldots,W7\}$ in Y$^<$ with the total order $W1 < W2 < \ldots  < W7$, each representing an inevitable waypoint. Our explanation algorithm detects five waypoints ($W1$-$W5$) as reachable depicted by Green ticks and reports $W6$ as the first unreachable waypoint shown by a Red cross in Figure~\ref{fig:waypoints}. The unreachability of $W6$ can lead the human expert to deduce that the rover's initial battery charge is insufficient to drive it past the ascending regions in $W4$ and $W5$. Figure~\ref{fig:waypoints1} shows the identified waypoints and an explanation on a 6$\times$6 warehouse domain for a planning problem where a robot needs to carry a consignment from its initial location to the goal location.
In the domain, Blue and Red cells are the initial and goal locations, respectively. Yellow cells have surfaces with oil-spillage and therefore robot has a greater rate of battery depletion in these cells. Grey cells are blocked. The Green cell is the only charging station.
The explanation algorithm identified 6 sub-problems $\{W1,W2,\ldots,W6\}$ for a planning problem depth bound of 12, and 4 sub-problems $W1$ to $W4$ for a depth bound of 17, respectively. The waypoints $W5$ and $W6$ are inevitable only when the depth bound is 12 since a path longer than 12 in length may not mandatorily visit these waypoints. Note that these vertices ($W5$ and $W6$) are not articulation points whereas the other waypoints ($W1-W4$) are articulation points of the warehouse grid graph. In both the problem instances, our algorithm reports that the robot can not reach the waypoint $W4$ under the dynamics which is the an explanation of unsolvability. A control engineer can deduce that the initial battery charge and the charge capacity of the robot is not sufficient to reach the waypoint directly or via the recharging station. Therefore, a higher charge capacity or a better placement of the charging station close to the waypoint $W4$ may be a work-around to make the task solvable.

%% file: related_works.tex
\section{Related works} \label{related_works}
Some notable works addressing the unsolvability of planning problems, mostly looked at verifying the unsolvability by generating certificates~\cite{DBLP:conf/aips/ErikssonH20}, \cite{DBLP:conf/aips/ErikssonRH17} or proofs \cite{DBLP:conf/aips/ErikssonRH18} rather than explaining the causalities of unsolvability of the planning problem. Such certificates or proofs of unsolvability are not enough to increase the human understandability of why the problem was unsolvable.
Most of these works focus on planning problems in discrete domains. Verifying the unsolvability of planning problems in hybrid systems comes with an additional challenge since these planning problems are undecidable in general \cite{ALUR19953}.
In~\cite{DBLP:journals/tecs/SarwarRB23}, authors provide an approach to addressing the unsolvability of a planning problem in hybrid domains by $\delta$-approximate bounded reachability analysis \cite{DBLP:journals/corr/GaoKCC14}. However, this work also verifies unsolvability rather than explaining it.
Few notable works that are directed towards explaining the unsolvability of a planning problem are, similarly, limited to classical planning problems.
Authors in~\cite{DBLP:conf/aips/GobelbeckerKEBN10} argue that excuses can be produced by counterfactual alterations to the original planning task such that the new planning task turns out to be solvable, and provides excuses for why a plan cannot be found.
In~\cite{DBLP:conf/aaai/EiflerC0MS20}, authors derive properties of a plan which could serve as explanations in case of unsolvability.
However, generating excuses, or deriving plan properties in terms of propositional formulas may not be enough to understand why a problem was unsolvable for complex domains like planning problems of hybrid systems which encode mixed discrete and continuous dynamics.
In \cite{DBLP:journals/jair/VasileiouYSKCM22}, an approach based on knowledge representation and reasoning has been applied to these domains. It provides explanations by finding a subset of the agent's knowledge base with which to reconcile the human knowledge base for explanations. However, it does not address unsolvability problems, rather, explains why a plan is feasible in a model.
A path-oriented reconciliation process between the agent and human models of hybrid systems is provided in~\cite{DBLP:conf/memocode/SarwarRB23}. It performs the reachability analysis along a path and uses the concept of irreducible infeasible sets (IIS) to generate explanations for unsolvability.

In this work, we propose to decompose an unsolvable planning problem into sub-problems motivated by the well-known insight that humans tend to break down sequential planning problems in terms of the sub-problems they need to achieve \cite{LT-1956,DBLP:journals/ai/VanLehn86}.
This has been a popular approach in many domains such as robotics \cite{24200} and AI \cite{DBLP:journals/ai/SuttonPS99} apart from planning \cite{DBLP:journals/jair/HoffmannPS04,DBLP:conf/ecai/LipovetzkyG12,DBLP:conf/aaai/RichterHW08}. 
\cite{DBLP:journals/jair/HoffmannPS04,DBLP:conf/ecai/LipovetzkyG12} find sub-problems for a solvable planning problem of the discrete domains in terms of ordered landmarks. Landmarks are facts given as propositional formulas that must be true at some point in every valid solution plan.
In~\cite{DBLP:conf/ijcai/SreedharanSSK19}, authors use hierarchical model abstractions to relax a planning problem until a solution can be found and looks for landmarks of this relaxed problem. They use these landmarks to identify the unachievable sub-problem for the planning problem. These works are in discrete domains. In contrast, our framework decomposes an unsolvable planning problem of hybrid domains into several smaller sub-problems by reducing it to an instance of longest common subsequence problem and consequently generating explanations using reachability analysis. 

%% file: conclusion.tex
\section{Conclusion} \label{conclusion}
In this paper, we explore the area of explaining the unsolvability of planning problems for hybrid systems by means of detecting the inevitable sub-problems that must be solvable in order for the bigger problem to be solvable.
We show a reduction from the problem of finding sub-problems and an ordering between them to finding the LCS of a finite set of path strings.  We present an explanation artifact through these sub-problems and by conducting reachability analysis. Results emphasize that our framework can efficiently identify inevitable sub-problems and the first infeasible one among them as an explanation for unsolvability of a planning problem. We believe that explanations reported by our algorithm can help a control engineer, an AI planner, or a human supervisor to comprehend the cause of unsolvability of the planning problem at hand.

%% file: Appendix.tex
\section{Appendix}\label{appendix}
\textbf{City-network domain:}
the context for this domain is a car that wants to reach a destination through a route network of a city. Figure~\ref{fig:city_example} shows the route network of the city. It has 10 important junctures. Blue-colored and green-colored routes connect these junctures. They respectively represent both-way and one-way traffic in the city. The direction of the traffic in green routes is shown with a directed arrow. The car is initially at juncture A. The juncture A is shown as the green-colored node in the figure. The car has a battery that depletes energy at a constant rate represented by a variable $b$. Initially, it has 20 units of battery charges available. Similarly, the juncture-to-juncture movement delay for the car is represented by a variable $d$ where each route has a different delay. From a juncture, the car can only move to the adjacent junctures following the route between them. The junctures H (red-colored), I (yellow-colored), and J (blue-colored) are the destinations of three different planning problems of the domain. The orange-colored nodes in the figure are the waypoints that appear in every source-to-destination path for a planning problem of A to H. The routes represent the discrete dynamics of the domain that captures the connectivity of the junctures of the city. The continuous dynamics of the domain involve energy depletion and the juncture-to-juncture movement delay of the car due to different traffic patterns.


\begin{figure}[htbp]
    \centering
    \includegraphics[width =0.66\textwidth]{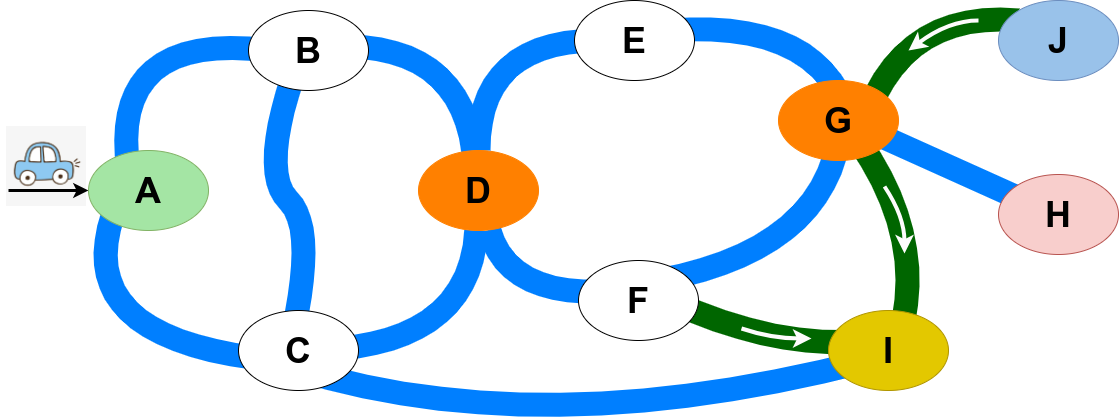}
    \caption{The city route network is depicted. Each node represents the important junctures of the city. The blue-colored routes represent both-way transportation between junctures. The green-colored routes represent one-way traffic. Initially, the car is at juncture A (shown as the green-colored node). The junctures I (yellow-colored), H (red-colored), and J (blue-colored) are the destinations of three different planning problems of the domain.
    The orange-colored nodes in the figure are the waypoints that appear in every source-to-destination path for a planning problem of A to H.}
    \label{fig:city_example}
\end{figure}

\noindent \textbf{Warehouse automation domain:} We present the context of warehouse automation \cite{DBLP:conf/memocode/SarwarRB23} where a robot operates to manage the inventories of the warehouse. The warehouse is divided into cells. The discrete dynamics here capture the connectivity of the cells along with the presence of objects in certain cells which are interpreted as obstacles through which the robot cannot move. The movement of the robot is restricted to one of its adjacent cells and movement to diagonal cells is prohibited. The continuous dynamics capture the battery charge depletion rate of the robot within a cell. Within each cell, the robot follows the dynamics particular to that cell. When the robot makes a transition from one cell to another, it starts to follow the dynamics of the new cell instantaneously. The robot is assigned the task of carrying a consignment to a designated cell while the number of cell visits is restricted to $\leq$ $D$ cells. The robot starts from the yellow-colored cell where the planning problem requires it to transport the black box to the goal cell (red-colored cell). The robot depletes its charge according to the cell dynamics while on the move. There is a charging station shown as a green-colored cell. The robot may visit this cell to recharge its battery. The grey-colored cells are blocked with obstacles. The robot is equipped with a rechargeable battery. The initial battery charge is 10 units. Each cell has a charge depletion rate of 2 units (modeling the continuous dynamics). 
In Figure~\ref{fig:warehouse-domain}, we have shown a representation of the warehouse automation domain. Now, consider the planning problem where the robot needs to carry the consignment to the goal from its initial location. Every feasible path for the robot must go through the cells marked with hatched lines (orange-colored) as shown in the figure. We consider these cells as landmarks for the planning problem. A landmark here means a cell that a robot must visit on its way to the goal.

\begin{figure}[htbp]
    \centering
    \includegraphics[width =0.66\textwidth]{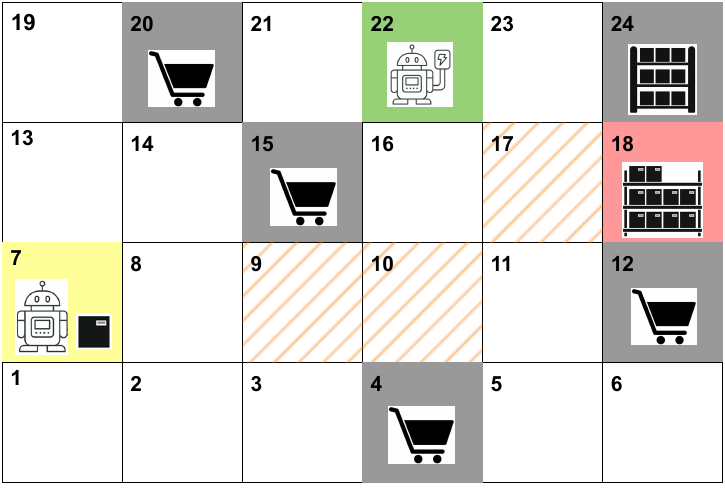}
    \caption{Warehouse automation domain.}
    \label{fig:warehouse-domain}
\end{figure}
